\documentclass[letterpaper, 10 pt, conference]{ieeeconf}  

\IEEEoverridecommandlockouts                              
\usepackage{graphics} 
\usepackage{epsfig} 
\usepackage{times} 
\usepackage{amsmath} 
\usepackage{amssymb}  

\usepackage{amsthm}

\usepackage{mathtools}

\usepackage{caption}
\usepackage{subcaption}

\usepackage{tcolorbox}
\DeclareMathOperator*{\argmin}{arg\,min}
\newtheorem{theorem}{\bf Theorem} 
 \newtheorem{remark}[theorem]{\bf Remark}

\title{\LARGE \bf
Training robust neural networks using Lipschitz bounds
}

\author{Patricia Pauli, Anne Koch, Julian Berberich, Paul Kohler and Frank Allg\"ower
\thanks{This work was funded by Deutsche Forschungsgemeinschaft (DFG,
German Research Foundation) under Germany’s Excellence Strategy - EXC
2075 - 390740016. The authors thank the International Max Planck Research
School for Intelligent Systems (IMPRS-IS) for supporting Patricia Pauli, Anne Koch, and Julian Berberich.}
\thanks{Patricia Pauli, Anne Koch, Julian Berberich, Paul Kohler and Frank Allg\"ower are with the Institute for Systems Theory and Automatic Control, University of Stuttgart, Germany
{\tt\small patricia.pauli@ist.uni-stuttgart.de}}      
}

\begin{document}

\maketitle
\thispagestyle{empty}
\pagestyle{empty}

\begin{abstract}
\textbf{Due to their susceptibility to adversarial perturbations, neural networks (NNs) are hardly used in safety-critical applications. One measure of robustness to such perturbations in the input is the Lipschitz constant of the input-output map defined by an NN. In this work, we propose a framework to train multi-layer NNs while at the same time encouraging robustness by keeping their Lipschitz constant small, thus addressing the robustness issue. More specifically, we design an optimization scheme based on the Alternating Direction Method of Multipliers that minimizes not only the training loss of an NN but also its Lipschitz constant resulting in a semidefinite programming based training procedure that promotes robustness. We design two versions of this training procedure. The first one includes a regularizer that penalizes an accurate upper bound on the Lipschitz constant. The second one allows to enforce a desired Lipschitz bound on the NN at all times during training. Finally, we provide two examples to show that the proposed framework successfully increases the robustness of NNs.}
\end{abstract}


\section{Introduction}
Neural networks (NNs) and deep learning have lately been successful in many fields \cite{krizhevsky2012imagenet,collobert2011natural}, where they are mostly used for classification and segmentation problems, as well as in reinforcement learning \cite{lecun2015deep}. The main advantages of NNs are that they can be trained straightforwardly using backpropagation and as universal function approximators they have the capability to represent complex nonlinearities. While NNs are powerful and broadly applicable they lack rigorous guarantees which is why they are not yet applied to safety-critical applications such as medical devices and autonomous driving. Adversarial attacks can easily deceive an NN by adding imperceptible perturbations to the input \cite{szegedy2013intriguing} which is a problem that has recently been tackled increasingly in a number of different ways such as adversarial training \cite{goodfellow2014explaining} and defensive distillation \cite{papernot2016distillation}. Another promising approach is to show that an NN is provably robust against norm-bounded adversarial perturbations \cite{wong2018provable, madry2017towards}, e.g. by maximizing margins \cite{tsuzuku2018lipschitz}, while yet another one is to use Lipschitz constants as a robustness measure that indicate the sensitivity of the output to perturbations in the input \cite{combettes2020lipschitz}. There are a number of other regression methods that provide guaranteed and optimized upper bounds on the Lipschitz constant such as nonlinear set membership predictions, kinky inference \cite{calliess2016lazily} and Lipschitz interpolation \cite{calliess2017lipschitz, maddalena2019learning}. Based on this notion of Lipschitz continuity, we propose a framework for training of robust NNs that encourages a small Lipschitz constant by including a regularizer or respectively, a constraint on the NN's Lipschitz constant. 

Trivial Lipschitz bounds of NNs can be determined by the product of the spectral norms of the weights \cite{szegedy2013intriguing} which is used during training in \cite{cisse2017parseval, gouk2018regularisation}. Similar to \cite{hein2017formal}, we use the Lipschitz constant as a regularization functional. However, they use local Lipschitz constants whereas we penalize the global one. In \cite{fazlyab2019efficient}, Fazlyab et al. propose an interesting new estimation scheme for more accurate upper bounds on the Lipschitz constant than the weights' spectral norms exploiting the structure of the nonlinear activation functions. Activation functions are gradients of convex potential functions, and hence monotonically increasing functions with bounded slopes, which is used in \cite{fazlyab2019efficient} to state the property of \emph{slope-restriction} as an incremental quadratic constraint and then formulate a semidefinite program (SDP) that determines an upper bound on the Lipschitz constant. In \cite{fazlyab2019efficient}, three variants of the Lipschitz constant estimation framework are proposed trading off accuracy and computational tractability. In this paper, we disprove by counterexample the most accurate approach presented in \cite{fazlyab2019efficient}, and we employ the other approaches for training of robust NNs. More specifically, we include the SDP-based Lipschitz bound characterization of \cite{fazlyab2019efficient} in the training procedure via an Alternating Direction Method of Multipliers (ADMM) scheme. We present two versions of the training method, (i) a regularizer rendering the Lipschitz constant small and (ii) enforces guaranteed upper bounds on the Lipschitz constant during training.

The main contributions of this manuscript are the two training procedures for robust NNs based on the notion of Lipschitz continuity, using a tight upper bound on the Lipschitz constant. In addition, we show that the method for Lipschitz constant estimation for NNs  that was recently proposed in \cite{fazlyab2019efficient} requires a modification for the least conservative choice of decision variables. This manuscript is organized as follows. In Section~\ref{sec:LipschitzEstimation}, we introduce Lipschitz constant estimation for NNs based on \cite{fazlyab2019efficient} but disprove their most accurate Lipschitz estimator. In Section~\ref{sec:ProposedApproach}, we present a training procedure with Lipschitz regularization and outline the setup of the optimization problem that is solved using ADMM. Subsequently, we introduce a variation of the proposed procedure that allows to enforce Lipschitz bounds on the NN and finally, we discuss the convergence and the computational tractability of the ADMM scheme. In Section~\ref{sec:simulationresults}, we provide two examples on which we successfully apply the proposed training procedures.

\section{Lipschitz constant estimation}\label{sec:LipschitzEstimation}
In this section, we briefly introduce robustness in the context of neural networks. We then state a method to estimate the Lipschitz constant of an NN based on \cite{fazlyab2019efficient} and finally argue why one of the methods for Lipschitz constant estimation proposed in \cite{fazlyab2019efficient} is incorrect.
\subsection{Robustness of NNs}
A robust NN should not change its prediction if the input is perturbed imperceptibly. To quantify robustness, a suitable robustness measure has to be defined. One definition is that perturbations from a norm-bounded uncertainty set may not alter the prediction. Alternatively, using probabilistic approaches, random  perturbations do not change the prediction with a certain probability. A third alternative is the Lipschitz constant, a sensitivity measure. A function $f:\mathbb{R}^n\rightarrow\mathbb{R}^m$ is globally Lipschitz continuous if there exists an $L\geq0$ such that
\begin{equation}\label{eq:Lip}
\left\lVert f(x)-f(y)\right\rVert\leq L \left\rVert x-y\right\rVert\quad\forall x,y\in\mathbb{R}^n.
\end{equation} 
The smallest $L$ for which \eqref{eq:Lip} holds is the Lipschitz constant $L^*$. If the input changes from $x$ to $y$, the Lipschitz constant gives an upper bound on how much the output $f$ changes. Hence, a low Lipschitz constant indicates low sensitivity which is equivalent to high robustness. In this work, we aim to minimize the Lipschitz constant or respectively bound the Lipschitz constant from above during training to increase the robustness of the resulting NN. 

Regularization, i.e., adding a penalty term to the objective function of the NN, is a prevalent measure in NN training in order to prevent overfitting. L2 regularization penalizes the squared norm of the weights and L1 regularization the weights' L1 norm. Bounding the weights counteracts the fit of sudden peaks and outliers, promotes better generalization, and smoothens the resulting NN \cite{krogh1992simple}. Furthermore, the product of the spectral norms of the weights provides a trivial bound on an NN's Lipschitz constant and consequently, L1 and L2 regularization improve the robustness of an NN in the sense of Lipschitz continuity. In this paper, we penalize a more accurate estimate of the Lipschitz constant, leading to a more direct and potentially more effective approach.

\subsection{Lipschitz constant estimation}\label{sec:Lip_est}
In the following, we outline a method to estimate bounds on the Lipschitz constant of multi-layer NNs exploiting the slope-restricted structure of the nonlinear activation functions, as it was shown in \cite{fazlyab2019efficient}. This method named LipSDP yields more accurate bounds than trivial bounds, i.e., the product of the spectral norms of the weights.

Continuous nonlinear activation functions $\varphi:\mathbb{R}\rightarrow\mathbb{R}$ can be interpreted as gradients of continuously differentiable, convex potential functions and consequently, they are slope-restricted, i.e., their slope is at least $\alpha$ and at most $\beta$, 
\begin{equation}\label{eq:slope-restriction}
\alpha\leq\frac{\varphi(x)-\varphi(y)}{x-y}\leq\beta\qquad\forall x,y\in\mathbb{R},
\end{equation}
where $0\leq\alpha<\beta<\infty$. Consider the vector of activation functions $\phi_i:\mathbb{R}^{n_{i}}\rightarrow\mathbb{R}^{n_{i}}, \phi_i(x^i)=[\varphi(x_1^i)\cdots\varphi(x_{n_i}^i)]^\top$ and a fully-connected feed-forward NN with $l$ hidden layers $f:\mathbb{R}^{n_0}\rightarrow\mathbb{R}^{n_{l+1}}$ described by the following equations
\begin{equation}\label{eq:NN}
\begin{split}
x^0&=x,\\
x^{i+1}&=\phi_{i+1}(W^ix^i+b^i),\quad i=0,\dots, l-1,\\
f(x)&=W^lx^l+b^l,
\end{split}
\end{equation}
where $W^i\in\mathbb{R}^{n_{i+1}\times n_i}$ are the weight matrices and $b^i\in\mathbb{R}^{n_{i+1}}$ are the biases of the $i$-th layers, $n_0,\cdots,n_{l+1}$ being the dimension of the input, the neurons in the hidden layers, and the output. For every neuron, the slope-restriction property \eqref{eq:slope-restriction} can be written as an incremental quadratic constraint. Weighting these quadratic constraints with $\lambda_{ii}\geq0$, or respectively, all activations with a diagonal weighting matrix
\begin{equation*}
T\in\mathcal{D}_n:=\{ T \in\mathbb{R}^{n\times n} \mid T=\sum_{i=1}^{n}\lambda_{ii}e_ie_i^\top, \lambda_{ii}\geq0\},
\end{equation*}
results in an incremental quadratic constraint for the stacked activations:
\begin{equation}\label{eq:Tcomb}
\begin{bmatrix} \tilde{x}-\tilde{y} \\
\phi(\tilde{x})-\phi(\tilde{y})
\end{bmatrix}^\top
\hspace{-0.08cm}\underbrace{\begin{bmatrix}
-2\alpha\beta T & (\alpha+\beta)T\\
(\alpha+\beta)T & -2T 
\end{bmatrix}}_{:=F(T)}
\hspace{-0.06cm}\begin{bmatrix} \tilde{x}-\tilde{y} \\
\phi(\tilde{x})-\phi(\tilde{y})
\end{bmatrix}
\hspace{-0.06cm}\geq\hspace{-0.06cm}0,
\end{equation}
for all $T\in\mathcal{D}_n$, $\tilde{x},\tilde{y}\in\mathbb{R}^n$ with $n=\sum_{i=1}^l n_i$. Herein, the inputs to all neurons are stacked up in one vector $\tilde{x}, \tilde{y}$, respectively, and the activation function $\phi:\mathbb{R}^n\to\mathbb{R}^n,~\phi(\tilde{x})=[\phi_1(x^1)^\top\cdots\phi_l(x^{l})^\top]^\top$ is then applied to the concatenated vector. 
We now formulate an SDP based on \cite{fazlyab2019efficient} that exploits \eqref{eq:Tcomb} for the estimation of an upper bound on the Lipschitz constant of the map characterized by the underlying NN. 
We therefore define
\begin{equation*}
A=
\begin{bmatrix}
W^0&\dots&0&0\\
\vdots&\ddots&\vdots&\vdots\\
0&\dots&W^{l-1}&0
\end{bmatrix},~
B=
\begin{bmatrix}
0 & I_n
\end{bmatrix}.
\end{equation*}

\begin{theorem}\label{thm:LipSDP}
Suppose there exist $L^2>0$, $T\in\mathcal{D}_{n}$ such that
\begin{align}\label{eq:LMIwithT}
P_l(L^2,T)\preceq 0,
\end{align}
where
\begin{equation*}
P_l(L^2,T):= \begin{bmatrix} A\\B\end{bmatrix}^\top F(T) \begin{bmatrix}A\\B\end{bmatrix} +\begin{bmatrix}-L^2I& 0 & 0\\ 0& 0 & 0 \\ 0 & 0 & {W^l}^\top W^l \end{bmatrix}\hspace{-0.08cm}.
\end{equation*}

Then, \eqref{eq:NN} is globally Lipschitz continuous with Lipschitz bound $L\geq L^*$.
\end{theorem}
The proof directly follows from A.4 in \cite{fazlyab2019efficient2}, the extended version of \cite{fazlyab2019efficient}. Note that this result differs from Theorem 2 in \cite{fazlyab2019efficient} as $T$ in our case is a diagonal matrix. As we show in Section \ref{sec:counter_example}, the result does not hold in general for the larger parametrization of $T$ provided in \cite{fazlyab2019efficient}. The smallest value for the Lipschitz upper bound is determined by solving the SDP
\begin{equation}\label{eq:SDP}
\min_{L^2, T}L^2~\text{s.\! t.}\quad P_l(L^2,T)\preceq0,~T\in\mathcal{D}_{n},
\end{equation}
where $T$ and $L^2$ serve as decision variables. 

\begin{remark}
In the case of one hidden layer the matrix $P_1$ reduces to
\begin{equation*}
P_1(L^2,T)\hspace{-0.08cm}=\hspace{-0.13cm}\begin{bmatrix} -2\alpha \beta {W^0}^\top TW^0-L^2 I\hspace{-0.25cm} & (\alpha+\beta){W^0}^\top T \\
    (\alpha+\beta)TW^0 &-2T+{W^1}^\top W^1\end{bmatrix}\hspace{-0.08cm}.
\end{equation*}
\end{remark}

\subsection{Counterexample for LipSDP with coupling}\label{sec:counter_example}
In \cite{fazlyab2019efficient}, three versions of the method LipSDP for Lipschitz constant estimation are stated that reconcile accuracy of the Lipschitz bound and computational complexity of the method by adjusting the number of the decision variables in the SDP. In this section, we give an illustrative counterexample to show that Theorems 1 and 2 in \cite{fazlyab2019efficient} are incorrect for the most accurate variant of LipSDP.

Theorem 2 in \cite{fazlyab2019efficient} resembles Theorem \ref{thm:LipSDP} of this manuscript with the difference that in \cite{fazlyab2019efficient} a set of symmetric coupling matrices
\begin{equation*}
\begin{split}
\mathcal{T}_n=&\{T\in\mathbb{R}^{n\times n} \mid T=\sum_{i=1}^{n}\lambda_{ii}e_ie_i^\top\\&+\sum_{1\leq i<j\leq n}\lambda_{ij}(e_i-e_j)(e_i-e_j)^\top, \lambda_{ij}\geq0\},
\end{split}
\end{equation*}
is introduced whereas we state Theorem \ref{thm:LipSDP} for a set of diagonal matrices $\mathcal{D}_n$.
In the following, we give a minimal counterexample to show that, as suggested in this manuscript, a further restriction of the class of $T$ is required. For that purpose, consider an NN with one hidden layer of size $n_1=2$, input and output size $n_0=n_2=1$, activation function $\tanh$ and weights and biases
\begin{align*}
W^0=\begin{bmatrix} -1 \\ -1\end{bmatrix},~b^0=\begin{bmatrix}-1\\  1\end{bmatrix},~W^1=\begin{bmatrix}-1 & 1\end{bmatrix},~b^1=-0.5.
\end{align*}
The resulting NN provides a good fit for the cosine function on $x\in[-\frac{\pi}{2},\frac{\pi}{2}]$ with a maximum deviation in the output of 0.0843. Therefore, the maximum slope of the cosine gives a good approximation of the Lipschitz constant of this NN, which is $\pm 1$ at $x=\pm\frac{\pi}{2}$, such that $L^*\approx 1$. However, the linear matrix inequality (LMI) \eqref{eq:LMIwithT} is feasible for arbitrarily small $L^2$ and $T=(e_1-e_2)(e_1-e_2)^\top\in\mathcal{T}_n$.
An arbitrarily small $L$ is obviously no upper bound on the Lipschitz constant of a cosine like function, thus contradicting Theorem 1 in \cite{fazlyab2019efficient}.

To understand why in this case the LipSDP method fails to provide an upper bound on the Lipschitz constant, we look into the coupling of the neurons accounted for by $T\in\mathcal{T}_n$.  Theorems 1 and 2 in \cite{fazlyab2019efficient} build on the assumption that \eqref{eq:Tcomb} holds for all $T\in\mathcal{T}_n$ (Lemma 1 in \cite{fazlyab2019efficient}). In the following, we show by counterexample that Lemma 1 in \cite{fazlyab2019efficient} is incorrect, i.e., that for a given slope-restricted function $\varphi$ there are $\tilde{x}, \tilde{y}\in\mathbb{R}^n$ and $T\in\mathcal{T}_n$ that violate \eqref{eq:Tcomb}.  
We choose $n_1=2$, $\varphi$ as the $ReLU$ function that is slope-restricted in the sector $[0,1]$, and $T=(e_1-e_2)(e_1-e_2)^\top\in\mathcal{T}_n$. Let $x=\begin{bmatrix}0 & 1\end{bmatrix}^\top$ and $y=\begin{bmatrix}-1.5 & 0 \end{bmatrix}^\top$. Then evaluating \eqref{eq:Tcomb} yields 
\begin{equation}
\begin{bmatrix}
\tilde{x}-\tilde{y}\\
\phi(\tilde{x})-\phi(\tilde{y})
\end{bmatrix}^\top
\begin{bmatrix}
0 & T \\ T & -2T
\end{bmatrix}
\begin{bmatrix}
\tilde{x}-\tilde{y}\\
\phi(\tilde{x})-\phi(\tilde{y})
\end{bmatrix}=-2<0,\nonumber
\end{equation}
which disproves Lemma~1 of \cite{fazlyab2019efficient}.

\section{Training robust NNs}\label{sec:ProposedApproach}
In Section \ref{sec:Lip_est}, we stated a method that provides certificates on an NN's Lipschitz constant. In this section, we employ these certificates to design a training procedure for robust NNs. The proposed approach allows us to directly regularize the Lipschitz constant during training, which is only possible indirectly in regularization methods such as L2 regularization. We present two versions of it, the first one allows for minimization of the upper bound on the Lipschitz constant and the second one allows to enforce a desired bound on the Lipschitz constant.
\subsection{Weights as decision variables}\label{sec:weights}
Eq. \eqref{eq:SDP} can be used to assess an NN's robustness after training, whereas in this manuscript, to promote robustness during training, we use Eq. \eqref{eq:SDP} to update the weights while minimizing the bound on the Lipschitz constant. Applying the Schur complement to \eqref{eq:LMIwithT} for $\alpha=0$, the LMI can be rearranged, yielding an equivalent LMI that is linear in $L^2$ and $W=(W^0,\cdots,W^{l})$, for fixed $T\in\mathcal{D}_{n}$:
\begin{equation} \label{eq:LMI}
\begin{split}
M_l(L^2,W):=&\begin{bmatrix} A\\B\end{bmatrix}^\top \begin{bmatrix}
0 & \beta T\\
\beta T & -2T 
\end{bmatrix} \begin{bmatrix}A\\B\end{bmatrix} \\ &+\begin{bmatrix}-L^2I& 0 & 0 & 0\\ 0& 0 & 0 & 0\\ 0 & 0 & 0 & {W^l}^\top \\ 0 & 0 & W^l & -I \end{bmatrix}\hspace{-0.08cm}\preceq0 
\end{split}
\end{equation}

\begin{remark}
For $\alpha>0$, the underlying constraint is not convex in $W$. Consequently, we cannot state LMI constraints for $\alpha>0$ and instead set $\alpha=0$. This is a conservative choice for some activation functions, yet the tight lower bound for the most common ones. E.g. for $\mathrm{ReLU}$ and $\tanh$ the tight bounds are $\alpha=0$ and $\beta=1$ and the sigmoid function is slope-restricted with $\alpha=0$ and $\beta=\frac{1}{4}$.
\end{remark}
\begin{remark}
For the single-layer case, $M_1$ conveniently reduces to
\begin{equation*} \label{eq:LMI2}
M_1(L^2,W)=\begin{bmatrix} -L^2 I & \beta{W^0}^\top T & 0\\
    \beta T W^0 &-2T & {W^1}^\top \\
    0 & W^1 & -I\end{bmatrix} .
\end{equation*}
\end{remark}

While the Lipschitz constant estimation scheme in \cite{fazlyab2019efficient} optimizes over $T$,  throughout the manuscript, we choose $T$ to be a fixed matrix. This introduces conservatism into the framework and necessitates a suitable choice for $T$ in order to keep the introduced conservatism to a minimum. For instance, the matrix $T$ may be determined from the Lipschitz constant estimation outlined in Section~\ref{sec:Lip_est} on the vanilla NN or the L2 regularized NN trained on the same problem.

\subsection{Lipschitz regularization}\label{sec:ADMM}
In general, NNs are trained on input-output data with the objective of minimizing a predefined loss, 
e.g. the mean squared error, cross-entropy, or hinge loss. We propose to not only minimize the NN's loss but also its Lipschitz constant. This yields an optimization problem with two separate objectives that can be solved conveniently using ADMM.

ADMM is an algorithm that solves optimization problems by splitting them into smaller subproblems that are easier to handle individually \cite{boyd2011distributed}. In order to apply the ADMM algorithm, the objective must be separable. The resulting subobjectives are then defined on uncoupled convex sets and are subject to linear equality constraints. The ADMM scheme solves the resulting optimization problem through independent minimization steps on the augmented Lagrangian of the optimization problem and a dual update step. The objectives at hand, i.e., the NN's loss and the Lipschitz bound, are indeed separable and defined on uncoupled convex sets. However, the problems are not completely independent and need to be connected through a linear constraint that requires the introduction of additional variables $\bar{W}=(\bar{W}^0,\dots,\bar{W}^l)$ of equal size as $W$. The loss of the NN $\mathcal{L}(W)$ is an explicit function of the weights $W$ and the Lipschitz bound $L$ depends on $\bar{W}$ through the LMI \eqref{eq:LMI}, yielding the following optimization problem:
\begin{equation}
\begin{split}\label{eq:opt_prob}
\min_{W,\bar{W},L^2} \qquad &\mathcal{L}(W)+\mu L^2 +\textbf{1}_{M_l}(L^2,\bar{W})\\
\text{s.\! t.} \qquad &W=\bar{W}
\end{split}
\end{equation}
where $\mu>0$ is a weighting parameter adjusting the trade-off between accuracy and robustness and
\begin{equation*}
\textbf{1}_M(L^2,\bar{W})=\begin{cases}
0& \text{if}~M_l(L^2,\bar{W})\preceq0 \\
\infty &\text{if}~M_l(L^2,\bar{W})\succ0
\end{cases}.
\end{equation*} 
is the indicator function. Applying the ADMM scheme to problem \eqref{eq:opt_prob}, results in the augmented Lagrangian function
\begin{equation*}
\begin{split}
\mathcal{L}_{\rho}(W,\bar{W},L^2,Y) := \mathcal{L}(W)+\mu L^2+\textbf{1}_{M_l}(L^2,\bar{W})\\+\mathrm{tr}(Y(W-\bar{W}))+\frac{\rho}{2}\left\lVert W-\bar{W}\right\rVert^2
\end{split}
\end{equation*}
with Lagrange multipliers $Y^i\in\mathbb{R}^{n_{i}\times n_{i+1}}$, $Y=(Y^0,Y^{1})$ and the penalty parameter $\rho>0$. The optimum for \eqref{eq:opt_prob} is then determined via the following iterative ADMM update steps:
\begin{subequations}\label{eq:update1}
\begin{align}
    W_{k+1}=&\argmin_{W} \mathcal{L}_{\rho}(W,\bar{W}_k,L_k^2,Y_k) \label{eq:loss_update}\\
    (L_{k+1}^2, \bar{W}_{k+1})=&\argmin_{L^2,\bar{W}} \mathcal{L}_{\rho}(W_{k+1},\bar{W},L^2,Y_k) \label{eq:Lip_update}\\
    Y_{k+1}=&Y_{k}+\rho(W_{k+1}-\bar{W}_{k+1}).
\end{align} 
\end{subequations}

For training of robust NNs, we carry out the corresponding updates consecutively until convergence. The loss function is optimized analytically using backpropagation (Eq. \eqref{eq:loss_update}) whereas the Lipschitz update step \eqref{eq:Lip_update} is an SDP, as implementation of the indicator function corresponds to an LMI constraint. Hence, the Lipschitz update step requires to solve an SDP in every iteration and thereby adds additional computations compared to the training of a vanilla NN.

\begin{remark}
It is possible to extend the framework and optimize over $L^2$, $T$, and $W$ at the same time which requires a second LMI constraint in \eqref{eq:opt_prob} and an additional update step in \eqref{eq:update1}, resulting in a multi-block ADMM scheme. This reduces conservatism but increases computation time.
\end{remark}

\subsection{Enforcing Lipschitz bounds}\label{sec:Lipschitz-bounds}
In Section \ref{sec:ADMM}, we suggested to minimize an upper bound on the Lipschitz constant of an NN. Using the proposed ADMM framework, it is also possible to enforce a desired upper bound on the Lipschitz constant during training of an NN. In that case, the Lipschitz constant is not minimized but instead set to a desired value $L_{\mathrm{des}}$. Judiciously, $L^2$ does not appear in the optimization objective of this setup
\begin{equation}\label{eq:feasibility}
\begin{split}
\min_{W,\bar{W}} \qquad &\mathcal{L}(W) +\textbf{1}_{M_l}(L_\textrm{des}^2,\bar{W})\\
\text{s.\! t.} \qquad &W=\bar{W}
\end{split}
\end{equation}
where $W$ and $\bar{W}$ serve as decision variables. For enforcement of Lipschitz bounds, we apply the ADMM algorithm as in \eqref{eq:update1} to \eqref{eq:feasibility} instead of \eqref{eq:opt_prob}, i.e., we set $L=L_\mathrm{des}$ instead of optimizing over $L^2$.


\begin{theorem}\label{thm:fesisbility}
When training a fully-connected NN with $l$ hidden layers \eqref{eq:NN} by executing the ADMM scheme \eqref{eq:update1} for \eqref{eq:feasibility}, $L_{\mathrm{des}}$ is an upper bound on the Lipschitz constant for the NN with weights $\bar{W}$ at all times during training.
\end{theorem}
\begin{proof}
Theorem \ref{thm:fesisbility} follows from the fact that, by design, the bound $L_{\mathrm{des}}$ on the Lipschitz constant is enforced in every iteration of the ADMM scheme, more specifically, in every Lipschitz update step \eqref{eq:Lip_update} by the LMI constraint $M_l(L_{\mathrm{des}}^2,\bar{W})\preceq0$ on the weights $\bar{W}$.
\end{proof}

The training procedure based on \eqref{eq:feasibility} allows to choose the value of the Lipschitz bound and to train NNs with Lipschitz guarantees. This way, a desired degree of robustness can be directly enforced. However, the choice of such a constraint on $L$ is always connected to the trade-off between accuracy and robustness, as the fit generally deteriorates when decreasing the Lipschitz constant constraint.
In addition, it is helpful to initialize the weight parameters appropriately which does not only accelerate training but may also facilitate a better fit. 

\subsection{Convergence}\label{sec:convergence}
ADMM was first introduced for optimization problems with convex subobjectives \cite{boyd2011distributed} and later, analyses of the ADMM scheme for non-convex objectives including further structural assumptions on the objective were formulated \cite{wang2019global}. In the proposed framework, the loss $\mathcal{L}(W)$ clearly is not convex and has no obvious structural properties which renders a thorough convergence analysis complicated and beyond the scope of this work. However, looking at the subproblems \eqref{eq:loss_update} and \eqref{eq:Lip_update} separately, we point out that for \eqref{eq:loss_update} gradient descent almost surely converges to local minima even for non-convex problems \cite{lee2016gradient}, and that \eqref{eq:Lip_update} is a semidefinite program with a unique minimizer. Thus adding the convex regularization term and the indicator function of a convex set to the optimization problem of NN training, that converges reliably, does not add complexity in the form of non-convexity to the optimization problem.

\subsection{Computational tractability}\label{sec:computational_tractibility}
Computational tractability and scalability of the proposed framework depend on the number of decision variables of the SDP. Generally, the complexity of SDP solvers scales cubically with the number of decision variables. 
Hence, for high-dimensional decision variables, solving an SDP is computationally more expensive than solving an unconstrained optimization problem using gradient descent. Therefore, the Lipschitz update step \eqref{eq:Lip_update} becomes the bottleneck of the proposed method as the number of neurons per hidden layers and the number of hidden layers, that together determine the size of the weights, increases.
For example for picture inputs as commonly used in classification problems, the input dimension is usually high, potentially leading to high computation times or computational intractability. Nevertheless, downscaling the input using convolutional or pooling layers provides an option to improve computation time on larger scale problems and makes the proposed method indeed a worthwhile one to infer robustness to a neural network and obtain guarantees on the Lipschitz bound, also on large-scale problems. In Section \ref{sec:simulationresults}, we show in an example that our method is applicable to MNIST, a typical benchmark classification problem.

In order to keep the number of Lipschitz update steps to a minimum, we advise to first fully train a neural network without regularizers or with standard regularizers, such as the L2 regularizer, which serves as an initialization of the matrix $T$ and the weights $W$, $\bar{W}$. Loosely speaking, our method provides a refinement of the pretrained neural network and allows to subsequently optimize robustness or impose robustness guarantees on the network by minimizing or respectively, by enforcing an upper bound on the Lipschitz constant.

\section{Simulation Results}\label{sec:simulationresults}
\begin{figure}
     \centering
     \vspace{0.2cm}
     \begin{subfigure}[b]{0.228\textwidth}
         \centering
         \includegraphics[width=\textwidth]{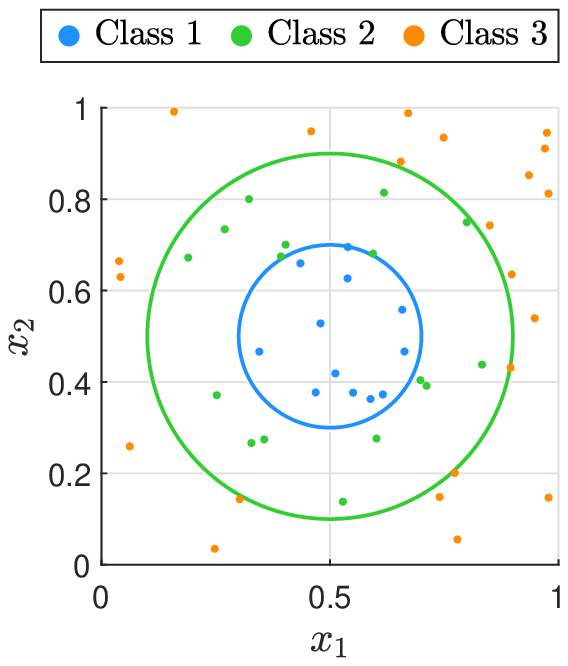}
         \caption{2D data}
         \label{fig:2D_setup}
     \end{subfigure}
     \begin{subfigure}[b]{0.242\textwidth}
         \centering
         \includegraphics[width=\textwidth]{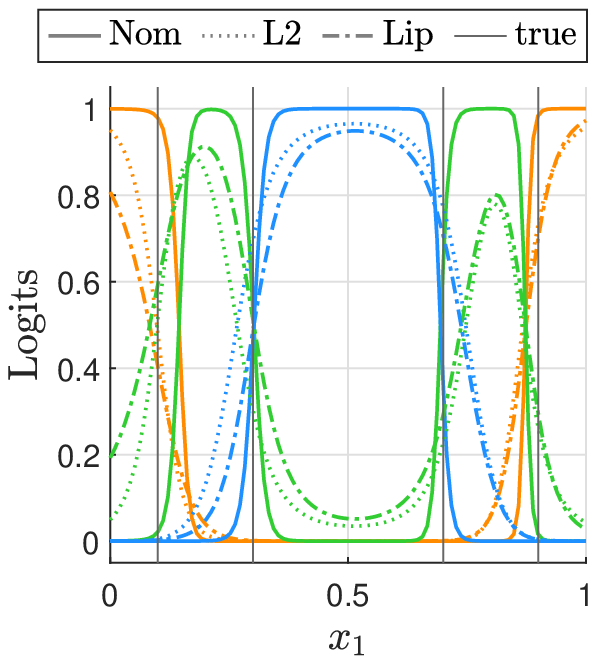}
         \caption{Logits at $x_2=0.5$}
         \label{fig:Logits}
     \end{subfigure}
     \caption{2D classification example.}
     \vspace{-0.2cm}
\end{figure}
In this section, we illustrate the benefits of the presented framework for training of robust NNs on a 2D toy example and on MNIST.
Our first illustrative example is a classification problem of 2D data with three classes, shown in Fig.~\ref{fig:2D_setup}. We design a feed-forward NN  with two hidden layers of $n_1=n_2=10$ neurons each, activation function $\tanh$ that is slope-restricted with $\alpha=0$, $\beta=1$, and the cross entropy loss (CEL) as the loss function. For comparison, we train three NNs, a vanilla NN, an NN with L2 regularization (L2-NN) for benchmarking and finally the Lipschitz regularized NN (Lipschitz-NN) according to Section \ref{sec:Lip_est}, wherein the NN loss update step~\eqref{eq:loss_update} is solved using stochastic gradient descent and the SDP~\eqref{eq:Lip_update} is solved using numerical SDP solvers \cite{Lofberg2004,mosek}.
Before training, we initialize the Lipschitz-NN with the L2 regularized NN with penalty parameter $\lambda=4\times10^{-3}$. The hyperparameters $\rho=0.25$ and $\mu=1\times10^{-5}$ are chosen such that for comparability the Lipschitz constants of the L2-NN and the Lipschitz-NN are roughly the same. The resulting cross-entropy losses, accuracies on test data and bounds on the Lipschitz constant, that are summarized in Table \ref{tb:CEL_Acc_L}, show that the nominal NN achieves a small CEL, yet a high Lipschitz bound of $242$. Comparing the two regularizers, we see that Lipschitz regularization here leads to both a lower Lipschitz bound $L$, hence higher robustness, and a lower CEL than L2 regularization. 
Even though, due to the trade-off between accuracy and robustness, the CEL of the Lipschitz-NN is higher than the CEL of the nominal NN, the accuracy of the Lipschitz-NN is not compromised. On the contrary, the Lipschitz-NN even provides the highest accuracy, as the nominal NN tends to overfit the data whereas the L2 regularized NN fails to provide a good fit in this example. Fig.~\ref{fig:Logits} shows a projection at $x_2=0.5$ of the logits resulting from the three NNs for all three classes (blue, green, and orange) onto the $x_1$ dimension. We clearly see the effect of a lower Lipschitz constant in the less steep slopes of the curves while the decision boundaries of the Lipschitz-NN remain accurate. The price to pay for the improved robustness is an increase in computation time (compare Table \ref{tb:CEL_Acc_L}), since training of the Lipschitz-NN requires several computationally involved ADMM iterations. Note that for low-dimensional problems the Lipschitz update step \eqref{eq:Lip_update} is faster than the loss update step \eqref{eq:loss_update}, yet it becomes computationally expensive for high-dimensional problems (cf. Section \ref{sec:computational_tractibility}).

\begin{table}[hb]
\begin{center}
\caption{CEL, accuracy, Lipschitz bound, training times}\label{tb:CEL_Acc_L}
\begin{tabular}{cccccc}
& & CEL & Accuracy & $L$ & $\bar{t}_\textrm{tr}$ (loss/Lip step)\\\hline
&Nom-NN & 0.07 & 88.66\% & 242 & 12.6s\\
2D ex. &L2-NN & 0.27 & 86.10\% & 69.5 & 16.9s\\
&Lip-NN & 0.22 & 90.56\% &  67.2 & 531s (41.3s/0.63s)\\ \hline
&Nom-NN   & 0.09 & 96.65\% & 96.6 & 56.4s\\
MNIST&L2-NN  & 0.26 & 90.58\% & 9.49 & 57.4s\\
&Lip-NN & 0.20 & 96.45\% & 8.74 & 2566s (365s/111s)\\ \hline
\end{tabular}
\end{center}
\end{table}

\begin{figure}
     \centering
     \vspace{0.2cm}
     \begin{subfigure}[b]{0.23\textwidth}
         \centering
         \includegraphics[width=\textwidth]{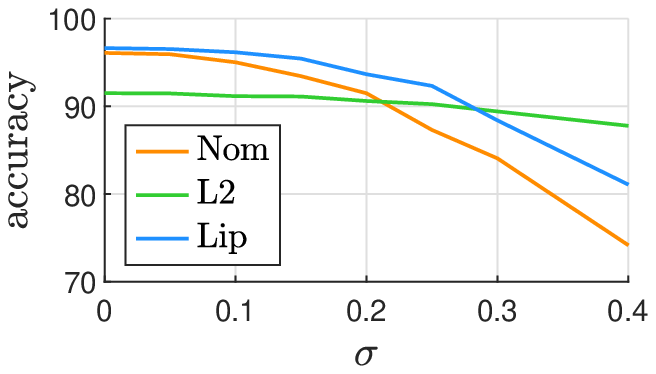}
         \caption{Gaussian noise}
         \label{fig:GaussNoise}
     \end{subfigure}
     \begin{subfigure}[b]{0.23\textwidth}
         \centering
         \includegraphics[width=\textwidth]{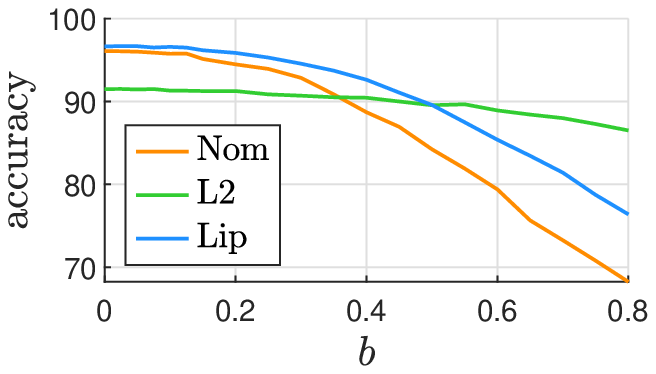}
         \caption{Uniform noise}
         \label{fig:UniformNoise}
     \end{subfigure}
     \caption{Accuracy on noise corrupted MNIST test data.}
     \label{fig:Noise}
     \vspace{-0.2cm}
\end{figure}

In our second example, we apply the Lipschitz regularization framework to the high-dimensional benchmark data set MNIST \cite{lecun1998mnist}.
We train a feed-forward NN with one pooling layer, one hidden layer with $n_1=50$ neurons, the activation function  $\tanh$, the cross entropy loss (CEL) and again, we train three NNs, a vanilla NN, an L2-NN, and a Lipschitz-NN with $\rho=0.25$ and $\mu=0.01$, initializing the Lipschitz-NN from the L2-NN with penalty parameter $\lambda=3\times10^{-3}$. The input dimension of the data is $28\times 28$ that is downscaled by the pooling layer to an input size of $14\times 14$. Note that the method also works on the $28\times 28$ data without a pooling layer, yet the Lipschitz step becomes significantly more time-consuming. From the resulting accuracies on test data and bounds on the Lipschitz constant shown in Table~\ref{tb:CEL_Acc_L}, we conclude that on MNIST our framework also finds an NN with a low Lipschitz constant but comparable accuracy to the nominal NN, while the accuracy of an L2-NN with a comparably low Lipschitz constant is significantly compromised. Fig. \ref{fig:Noise} shows the evaluation of the three NNs on noise corrupted data, that was created by adding Gaussian noise $\mathcal{N}(0,\sigma^2)$ and respectively, uniform noise $\mathcal{U}(-b,b)$ to the $(0,1)$-normalized MNIST data. Advantages of the Lipschitz-NN become apparent for Gaussian noise with low standard deviation and noise from a narrow uniform distribution.

Altogether, the results show that Lipschitz regularization can be used to effectively train robust NNs while trading off robustness and accuracy.
Code to reproduce the examples can be found at \emph{https://github.com/st157640/Training-robust-neural-networks-using-Lipschitz-bounds.git}.

\section{Conclusion}
We proposed a framework for training of multi-layer NNs that encourages robustness, by both considering Lipschitz regularization and by enforcing Lipschitz bounds during training. The underlying SDP \cite{fazlyab2019efficient} estimates the upper bound on the Lipschitz constant more accurately than traditional methods as it exploits the fact that activation functions are slope-restricted. We designed an optimization scheme based on this SDP that trains an NN to fit input-output data and at the same time increases its robustness in terms of Lipschitz continuity. We used ADMM to solve the underlying optimization problem and to therein conveniently incorporate the trade-off between accuracy and robustness. In addition, we presented a variation of the framework that allows for bounding the Lipschitz constant by a desired value, i.e., training NNs with robustness guarantees. We successfully tested our method on two examples where we benchmarked it with L2 regularization.

Next steps include the application of our method to control problems by using LMI constraints to verify and enforce properties, such as closed-loop stability, on feedback interconnections that include NN controllers. Also, we plan to explore alternatives for the ADMM algorithm that solve the underlying optimization problem in an accelerated manner. In addition, for benchmarking purposes, we plan to compare the proposed methods to other training procedures that improve robustness.

\bibliographystyle{IEEEtran}
\bibliography{references}

\end{document}